\documentclass[letterpaper, 10 pt, conference]{ieeeconf}  

\IEEEoverridecommandlockouts                              
\usepackage{graphicx} 
\usepackage{amsmath} 
\usepackage{amssymb}  
\usepackage{url}
\usepackage{color}
\usepackage{mathtools}
\mathtoolsset{showonlyrefs}
\newtheorem{theorem}{Theorem}

\newtheorem{remark}{Remark}
\newtheorem{proposition}{Proposition}

\usepackage{algorithm}
\usepackage{algpseudocode}
\usepackage[table]{xcolor}
\usepackage{colortbl}
\title{\LARGE \bf
Adversarial Robustness of Deep State Space Models for Forecasting
}
\author{Sribalaji C. Anand and George J. Pappas
\thanks{This research is supported by the Swedish Research Council grant 2024-00185. The authors are with the Department of Electrical and Systems Engineering, University of Pennsylvania, United States (email: \{sri03,pappasg\}@seas.upenn.edu). Sribalaji C. Anand is also affiliated with KTH Royal Institute of Technology, Sweden (email: srca@kth.se).}%
}
\begin{document}
\maketitle
\thispagestyle{empty}
\pagestyle{empty}
%
%
\begin{abstract}
State-space model (SSM) for time-series forecasting have demonstrated strong empirical performance on benchmark datasets, yet their robustness under adversarial perturbations is poorly understood. We address this gap through a control-theoretic lens, focusing on the recently proposed \emph{Spacetime} SSM forecaster. We first establish that the decoder-only \emph{Spacetime} architecture can represent the optimal Kalman predictor when the underlying data-generating process is autoregressive - a property no other SSM possesses. Building on this, we formulate robust forecaster design as a Stackelberg game against worst-case stealthy adversaries constrained by a detection budget, and solve it via adversarial training. We derive closed-form bounds on adversarial forecasting error that expose how open-loop instability, closed-loop instability, and decoder state dimension each amplify vulnerability - offering actionable principles towards robust forecaster design. Finally, we show that even adversaries with no access to the forecaster can nonetheless construct effective attacks by exploiting the model's locally linear input-output behavior, bypassing gradient computations entirely. Experiments on the Monash benchmark datasets highlight that model-free attacks, without any gradient computation, can cause at least $33\%$ more error than projected gradient descent with a small step size. 
\end{abstract}
%
%
\section{Introduction}
Time series modeling (TSM) is a well-established problem that requires models to efficiently forecast over long horizons and finds applications in diverse domains including finance~\cite{andersen2003modeling}, power systems~\cite{amini2016arima}, climate science~\cite{mudelsee2019trend}, among others. With increasing data availability and computational power, purely data-driven and machine learning-aided TSM has become an active research area~\cite{jin2024survey}.

Various machine learning architectures have been adopted in the literature for effective TSM, including Convolutional/Recurrent Neural Networks (CNN/RNN)~\cite{bai2018empirical}, 
Transformers~\cite{das2024decoder}, and deep state-space models (SSMs)~\cite{rangapuram2018deep}. Although there has been a surge of Transformer-based solutions for TSM, empirical evidence suggests that a simple one-layer neural network outperforms sophisticated Transformer-based models, often by a large margin~\cite{zeng2023transformers}.

In parallel, a deep SSM called \emph{Spacetime}~\cite{zhang2023effectively} was recently proposed for TSM and shown to outperform neural networks and Transformers on benchmark datasets (see \cite[Table~1]{zhang2023effectively}). \emph{Spacetime} was developed based on the SSM in~\cite{gu2021efficiently}, which was the precursor to the well-known language model \emph{Mamba}~\cite{gu2024mamba}. In this paper, we propose for the first time a framework to robustify the \emph{Spacetime} model against stealthy adversaries. In particular, the contributions are:
\begin{enumerate}
\item We establish that when the underlying data-generating mechanism is autoregressive, the decoder-only \emph{Spacetime} model can represent the optimal Kalman predictor under mild conditions (Proposition~\ref{prop:Kalman}).
\item We robustify the \emph{Spacetime} model against worst-case adversaries by formulating a robust optimization problem in \eqref{eq:problem} and solving it via adversarial training.
\item We quantify the forecasting error induced by adversarial perturbations and characterize its dependence on the \emph{Spacetime} model parameters, providing insights for robust forecaster design (Proposition~\ref{prop:error}).
\item We demonstrate that when adversaries lack access to the forecaster model, data-driven attacks can compromise forecasting performance with relative ease, highlighting the model's vulnerability (Theorem~\ref{thm:DDA}).
\item We validate our framework through experiments on the Monash benchmark time series datasets, depicting that: (a) detector-constrained adversarial training can yield up to $10\%$ reduction in adversarial MAE, and (b) model-free attacks can cause at least $33\%$ more MAE than projected gradient descent with a small step size.
\end{enumerate}

This paper is one of the first to study the robustness of SSM-based forecasting models. However, the impact of attacks on other forecasters has been studied in the literature. For instance, \cite{pmlr-v258-liu25l} underscores the impact of adversarial attacks on Transformer-based forecasters and proposes a gradient-free attack scheme based on model queries. The work~\cite{lin2024backtime} studies the impact of stealthy poisoning attacks and develops robust models through adversarial training. The paper~\cite{gallagher2022investigating} examines the effect of Fast Gradient Sign Method (FGSM) attacks on CNN models used for time series classification.
The paper~\cite{wu2022small} develops attacks on time series predictions using gradients, with extensions to constrained perturbation scenarios. The paper~\cite{zizzo2020adversarial} considers generating attacks against LSTM detectors.

While the aforementioned works focus on TSM, recent work has begun examining adversarial robustness of SSMs in other domains. The paper~\cite{qi2024exploring} analyzes SSMs under adversarial perturbations, concluding that input-dependent selective SSMs~\cite{gu2024mamba} may face the problem of error explosion. The effect of bit-flip attacks on SSMs is studied in~\cite{das2025rambo}, demonstrating that flipping a single critical bit can reduce accuracy from $74.64\%$ to $0\%$. Similarly, the vulnerability of visual SSMs against adversarial attacks was studied in~\cite{lee2024badvim}.

Thus, while substantial work exists on adversarial attacks in TSM and on SSM robustness in other domains, the robustness of SSM-based forecasters remains unexplored. Moreover, the problem has not been examined from a control-theoretic perspective, which is the focus of this paper.

The remainder of this paper is organized as follows. We formulate the problem in Section~\ref{sec:PF}. 
In Section~\ref{sec:methods}, we introduce the \emph{Spacetime} model and provide a control-theoretic analysis of the SSM.
In Section~\ref{sec:design}, we construct a robust forecaster, and provide a robustness analysis. 
In Section~\ref{sec:DDA}, we propose model-free attack strategies and we conclude the paper in Section~\ref{sec:con}.
Experimental validation on benchmark datasets is provided throughout the paper.

%
%
%
\section{Problem Formulation}\label{sec:PF}
In this section, we introduce the preliminaries and formulate the problem. A pictorial representation of the problem setup is given in Fig.~\ref{fig:problem}.
\begin{figure}
    \centering
    \includegraphics[width=7cm]{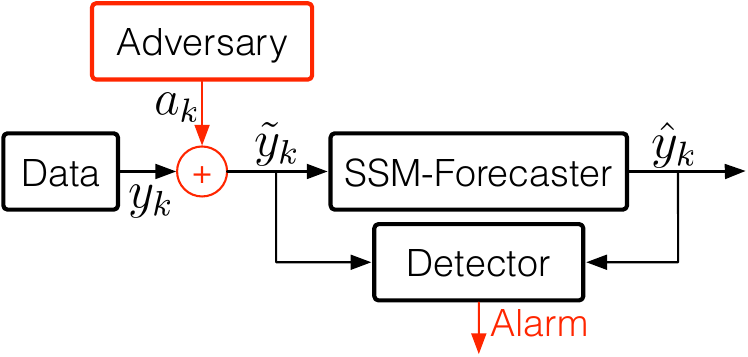}
    \vspace{-10pt}
    \caption{Problem setup: the adversary (red) injects attack signal into the data stream, producing corrupted input $\tilde{y}_k$ to the SSM-Forecaster. The detector uses $\tilde{y}_k$ and $\hat{y}_k$ to raise an alarm.}
    \vspace{-15pt}
    \label{fig:problem}
\end{figure}
\subsection{Data-generating mechanism and forecaster model} 
We consider a scalar time-series $y_k \in \mathbb{R}$, $k \in \mathbb{Z}_+$, where $y_k$ denotes the value at time step $k$, and $y$ denotes the entire sequence. 
Suppose that we have access to a large amount of attack-free historical data $\mathcal{I} \triangleq \left\{y_k\right\}_{k=0}^N$, $N\gg 0$.
Using $\mathcal{I}$, our objective is to construct a forecaster of the form:
\begin{equation}\label{eq:f}
    \hat{y}_{k+1:k+h}^{(k)} = f\left({y}_{k-\ell+1:k}\right),
\end{equation}
where $h \geq 1$ is the forecasting horizon, $\ell \geq 1$ is the look-back window, $\hat{y}_{j}^{(k)}$ denotes the prediction for time $j$ made at time $k$, and $f$ denotes the forecaster obtained using $\mathcal{I}$. 
%
\subsection{Attack scenario} 
During runtime, the input to the forecaster may be corrupted by malicious adversaries. Specifically, we consider that the input to the forecaster is corrupted as:
\begin{equation}\label{eq:attacks}
\tilde{y}_k = y_k + a_k,
\end{equation}
where $\tilde{y}_k$ is the attacked data received during runtime, $y_k$ is the true data, and $a_k$ is the attack signal injected by the adversary. For notational simplicity, we denote the attacked sequence as $\tilde{y} = y + a$, where $a$ is the attack sequence.
%
\subsection{Attack detector and false alarm rate} 
Although the forecaster may not know the attack magnitude or duration, to detect such attacks during runtime, an attack detector is employed as follows:
\begin{align}
\text{Detector:}& \begin{cases}
        z_k> \delta & \text{alarm}\\
        z_k\leq \delta & \text{No alarm}
    \end{cases},\label{eq:detector}\\
    z_k &= g(\tilde{y}_k,\bar{y}_k),\;\;\bar{y}_{k} = \displaystyle \frac{1}{h} \sum_{i=1}^{h} \hat{y}_{k}^{(k-i)}, \label{eq:att:pre}
\end{align}
where $\delta \in \mathbb{R}_+$ is the detection threshold, $z_k$ is the detection statistic, $\hat{y}$ are the predictions made using possibly attacked data $\tilde{y}$,
and $\bar{y}_k$ is the prediction average over $h$ different predicted values of $y_k$. Here, $\delta$ is a design parameter. If $\delta$ is small, the false alarm rate (FAR) will be high, which is detrimental. Similarly, if $\delta$ is large, the adversary can inject attacks of larger magnitudes that remain undetected. Thus, $\delta$ is designed to yield an acceptable FAR, denoted by $\alpha$. Here, the FAR is defined as $\text{FAR} = \mathbb{P}(z_k > \delta \mid a \equiv 0)$.
\begin{remark}
Suppose the data-generating mechanism is an LTI system. Let $h=1$, $\ell=1$, $z_k = \sigma_e^{-1} e_k^2$, $e_k = \tilde{y}_k-\bar{y}_k$, where $\sigma_e$ is the forecasting error variance. Then the detector \eqref{eq:detector} represents the $\chi^2$ detector. The threshold $\delta$ that yields a given FAR $\alpha$ (asymptotically) is given by \cite[(14)]{anand2025feasibility}. $\hfill \triangleleft$
\end{remark}
\subsection{Attacker knowledge, constraints, and objective}
In this paper, we assume that the adversary has access to $f(\cdot)$, $\delta$, $g(\cdot)$, $\alpha$, and $\mathcal{I}$. In general, the adversary may not have access to such information, but this assumption allows us to defend against the worst-case adversary. Suppose the adversary constructs an attack signal $a_k$ that does not increase the FAR $\alpha$, in which case the forecaster may not detect the presence of an adversary. We define such attack signals that do not raise the FAR as \emph{stealthy attacks}. Given $\alpha$, we denote the set of all stealthy attacks as $\mathcal{S}^{\alpha}$. 

We next consider an adversary injecting stealthy attacks to maximize the prediction error. Let $H \triangleq N-h-\ell+1$, then the attack policy can be obtained by solving:
\begin{equation}\label{opt:adversary}
\begin{aligned}
& \underset{a \in \mathcal{S}^{\alpha}}{\sup} \;\; \frac{1}{H} \sum_{j=\ell}^{N-h} Q_{f}(y_{j-\ell + 1:j+h}, a_{j-\ell+1:j})\\
&\scalebox{0.95}{$Q_{f}(y_{j-\ell + 1:j+h}, a_{j-\ell+1:j}) = \left\| f(\tilde{y}_{j-\ell+1:j}) - y_{j+1:j+h} \right\|_2^2,$}
\end{aligned}
\end{equation}
where $Q_f(y,a)$ is the squared error caused by a given attack vector $a$ against a forecaster $f(\cdot)$ with input $y$.
\subsection{Robust forecaster and problem definition}
To defend against the worst-case adversary in \eqref{opt:adversary}, we aim to construct a robust forecaster that reduces the mean squared error (MSE) of predictions in the presence of attacks. Such a robust forecaster can be obtained by solving:
\begin{equation}\label{eq:problem}
\begin{aligned}
    \underset{f \in \mathcal{F}}{\inf} \; \underset{a \in \mathcal{S}^{\alpha}}{\sup} \;\; \frac{1}{H} \sum_{j=\ell}^{N-h} Q_{f}(y_{j-\ell + 1:j+h}, a_{j-\ell+1:j}),
\end{aligned}
\end{equation}
where $\alpha \in (0,1)$ is the nominal FAR and $\mathcal{F}$ is the set of all forecasters which can be realized using the \emph{Spacetime} SSM (see next section for more details).

The forecaster design problem \eqref{eq:problem} can be interpreted as a zero-sum Stackelberg game where the forecaster is the leader and the adversary is the follower. The forecaster commits to a model $f(\cdot)$ first, anticipating the worst-case adversarial response. By designing the forecaster to minimize the MSE under this worst-case attack, we obtain a robust model that performs well even when the adversary best-responds to the deployed forecaster. The remainder of this paper aims to solve the optimization problem~\eqref{eq:problem}.
\begin{remark}\label{rem:energy}
We note that most adversarial attack formulations in machine learning assume that attack energy is norm-bounded \cite{madry2018towards,yoon2022robust}. 
However, in this paper, rather than constraining the attacker's energy budget, we assume the attacker remains stealthy with respect to measurable detector signals. From both a practical and worst-case attack formulation perspective, our problem formulation is more realistic: attackers in real-world scenarios are typically constrained by detectability rather than by arbitrary energy bounds. $\hfill \triangleleft$
\end{remark}
\section{Spacetime Model}\label{sec:methods}
In this section, we present a brief overview of the forecaster model, and provide a control-theoretic analysis of the model. We also introduce the benchmark dataset to depict the performance of the forecaster model.
\subsection{Forecaster model} 
As mentioned before, in this paper, we use the SSM-based forecaster \emph{Spacetime}. A detailed overview of the \emph{Spacetime} model can be found in \cite{zhang2023effectively}; however, we present a brief overview to keep the presentation self-contained with the help of a pictorial representation in Figure~\ref{fig:spacetime}.

\begin{figure}[t]
    \centering
    \includegraphics[width=6cm]{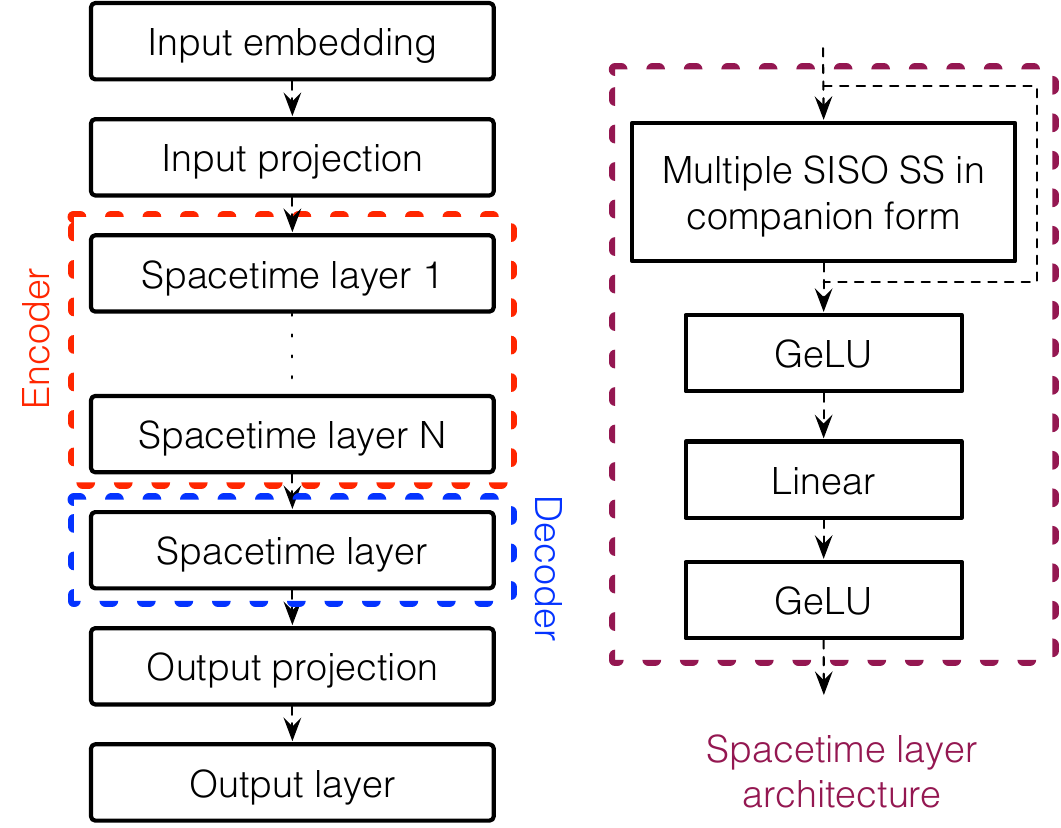}
    \caption{Spacetime architecture (left) and layer components (right). Here GeLU represents a Gaussian Error Linear Unit activation function, and Linear denotes a linear activation function.}
    \vspace{-15pt}
    \label{fig:spacetime}
\end{figure}

The \emph{Spacetime} model consists of an input embedding (to convert time-series data to vector representations), input projections (for dimension matching), a stack of \emph{Spacetime} layers (for encoding and decoding), an output projection (for dimension matching), and an output layer. Each \emph{Spacetime} layer comprises multiple Single-Input Single-Output (SISO) State-Space (SS) matrices in controllable canonical form with skip connections, whose outputs are mixed using a feed-forward network with GeLU activation function. Each encoder layer processes an input time series as a sequence-to-sequence map. The decoder layer takes the encoded sequence as input and outputs a predicted sequence. Unlike the encoder layers, which use skip connections, the decoder \emph{Spacetime} layer has no activation functions and no skip connections.

A key advantage of the \emph{Spacetime} model is its ability to predict its inputs (in the embedded domain) alongside the outputs (see \cite[(15)]{zhang2023effectively}). This enables the model to recurrently generate its own future inputs at inference time, leading to auto-regressive predictions without being constrained to fixed-horizon predictions~\cite{zeng2023transformers}.

Thus we now assume access to an oracle \emph{Spacetime} that produces a forecaster $f(\cdot) = \text{\emph{Spacetime}}(\mathcal{I},\cdot)$.
We next depict the efficacy of the \emph{Spacetime} model.
\begin{remark}\label{rem:zero_dynamics}
The detector-constrained formulation in~\eqref{eq:problem} admits stronger attacks than energy-constrained formulations. To see this, consider that the weights of the MLP block are chosen so that they represent the identity function, so that the \emph{Spacetime} model reduces to a linear map. If the encoder possesses unstable zeros, a zero dynamics attack~\cite{teixeira2015secure} can drive the internal states to large magnitudes while keeping the encoder output (and hence the detection statistic) arbitrarily small. Once the states grow sufficiently large causing bit overflow or the attack stops, the forecasting output degrades. The formulation in~\eqref{eq:problem} naturally accommodates such attacks, whereas such attacks violate a finite energy bound and are infeasible under energy-constrained formulations. 
$\hfill \triangleleft$
\end{remark}
\subsection{Control-theoretic analysis of the \emph{Spacetime} model}
In this section, we first present a result emphasizing the efficacy of the \emph{Spacetime} model. To this end, we recall a result from \cite{zeng2023transformers}.
\begin{proposition}\label{prop:spacetime:benefit}
Let the underlying data-generating mechanism be a noiseless auto-regressive (AR) process:
\begin{equation}\label{eq:prop}
    y_k = \sum_{i=1}^{p} \phi_i y_{k-i}, \quad p \in \mathbb{Z}_+, \quad p \geq 1.
\end{equation}
Then no class of Linear SSMs~\cite{gu2021efficiently}, except \emph{Spacetime}, can exactly represent \eqref{eq:prop}. $\hfill \square$
\end{proposition}

Proposition~\ref{prop:spacetime:benefit} states that only \emph{Spacetime} can accurately represent an AR process, which is a common model for time-series data~\cite{box2015time}. Next, we show that the \emph{Spacetime} model can represent an optimal predictor for the autoregressive system \eqref{eq:prop}, thanks to its inherent linear structure. To show this optimality, we first rewrite \eqref{eq:prop} as:
\begin{equation}
\bar{x}_{k+1} = A\bar{x}_k,\; y_k = C \bar{x}_k, \;  A \triangleq 
\begin{bmatrix}
\phi_1 & \phi_{2}  & \cdots & \phi_p\\
 1 & 0 &\cdots & 0\\
 \vdots & \ddots & \ddots & \vdots\\
 0 & \cdots & 1 & 0
\end{bmatrix}
\end{equation}
where $C \triangleq \begin{bmatrix}
1 & 0  & \cdots & 0
\end{bmatrix}$, $\bar{x}_k = \begin{bmatrix}
    y_{k-1} & \dots & y_{k-p}
\end{bmatrix}^\top$.

For a stable system ($\rho(A)<1$), the steady-state optimal one-step-ahead predictor is the Kalman predictor. We next show that the decoder-only \emph{Spacetime} model can represent any Luenberger-type observer, which includes the optimal Kalman predictor as a special case. 
\begin{proposition}\label{prop:Kalman}
Let the data-generating mechanism be:
\begin{equation}\label{eq:prop:2}
\begin{aligned}
    \bar{x}_{k+1} &= A\bar{x}_k + \omega_k, \quad \omega_k \sim \mathcal{N}(0,\Sigma_w)\\
    y_k &= C\bar{x}_k + v_k, \quad v_k \sim \mathcal{N}(0,\Sigma_v),
\end{aligned}
\end{equation}
where $\rho(A) < 1$ and the pair $(A,C)$ is observable. Consider a steady-state observer making one-step-ahead predictions:
\begin{equation}\label{eq:Luen:O}
    \hat{x}^o_{k+1} = A\hat{x}^o_{k} + L(y_k - C\hat{x}^o_k), \quad \hat{y}^o_{k+1} = C\hat{x}^o_{k+1},
\end{equation}
where $\hat{x}^o_{k+1}$ and $\hat{y}^o_{k+1}$ denote the predicted state and output, respectively, and $L$ is the observer gain. Consider a decoder-only \emph{Spacetime} model, where the weights of the MLP block are chosen so that they represent the identity function, making one-step-ahead predictions: 
\begin{equation}\label{eq:ST:O}
    \hat{y}^s_{k+1} = f(u_k^s)
\end{equation}
where 
$\hat{y}^s_{k+1}$ is the predicted output, and $u_k^s = y_k$ is the input sequence.
If the pair $(A, L)$ is controllable, then there exist a spacetime model $f(\cdot)$
such that $\|\hat{y}^s_{k+1} - \hat{y}^o_{k+1}\| \leq \epsilon$ for arbitrarily small $\epsilon > 0$.
\end{proposition}
\begin{proof}
For a decoder-only \emph{spacetime} model, under the stated assumptions, the predictions can be written as: $\hat{x}^s_{k+1} = A^s\hat{x}^s_{k} + B^su_k^s$, $\hat{y}^s_{k+1} = C^s\hat{x}^s_{k+1}$, where $\hat{x}^s_{k}$ is the decoder state. The proof then follows by showing that there exist matrices $(A^s, B^s, C^s)$ and an initial condition $\hat{x}^s_0$ such that $\hat{y}^s_{k+1} = \hat{y}^o_{k+1}$. To this end, let $A^s = A - LC$, $B^s = L$, $C^s = C$, and $\hat{x}^s_{0} = \hat{x}^o_{0}$. With this choice, the dynamics in \eqref{eq:ST:O} and \eqref{eq:Luen:O} become identical, satisfying $\hat{y}^s_{k+1} = \hat{y}^o_{k+1}$. For the dynamics to be realizable by the decoder-only \emph{Spacetime} model, we must show that the matrices $A^s$, $B^s$, $C^s$ can be represented in controllable canonical form. Since the time series is scalar, the state-space system in \eqref{eq:ST:O} is SISO. For a SISO system, the matrices can be represented in controllable canonical form if and only if the pair $(A^s, B^s) = (A - LC, L)$ is controllable.

We prove that $(A - LC, L)$ is controllable by contradiction. Assume $(A, L)$ is controllable but $(A - LC, L)$ is not controllable. Then there exists $v \neq 0$ and $\lambda \in \mathbb{C}$ such that
\begin{equation}\label{eq:2}
    \scalebox{0.88}{$v^\top \begin{bmatrix} \lambda I - (A - LC) & L \end{bmatrix} = 0 \implies v^\top \begin{bmatrix} \lambda I - A & L \end{bmatrix} = 0$}
\end{equation}
where the implication follows since $v^\top L = 0$ implies $v^\top LC = 0$. Condition~\eqref{eq:2} holds if and only if $(A, L)$ is not controllable, contradicting our assumption. Therefore, $(A - LC, L)$ is controllable, completing the proof.
\end{proof}
\begin{remark}
The controllability condition in Proposition~\ref{prop:Kalman} is mild in practice. For instance, an AR($3$) process with $\phi_1 = 0.3$, $\phi_2 = 0.5$, $\phi_3 = 0.2$, $\Sigma_v = 0.1$, and $\Sigma_w = 10^{-2}I_3$, has a Kalman gain $K = [0.16, 0.20, 0.17]^\top$, and one can immediately verify that $(A, K)$ is controllable. $\hfill\triangleleft$
\end{remark}

Thus, we have shown that the decoder-only \emph{Spacetime} model can represent any Luenberger-type predictor, including the steady-state optimal Kalman predictor. In comparison, a transformer architecture can represent a Kalman filter~\cite{goel2024can} up to a small additive error that is bounded uniformly in time. Our result establishes that the \emph{Spacetime} architecture is optimal in a well-defined sense; namely, it can represent the best possible linear predictor for autoregressive data-generating processes. We next depict the performance of the model using a benchmark dataset.

\subsection{Experiments}\label{dataset:introduce}
In this section, we demonstrate the efficacy of the \emph{Spacetime} model using a benchmark dataset. In particular, we use the electricity consumption dataset from~\cite{electricityloaddiagrams}, which comprises hourly electricity consumption measurements (in kW) from $321$ clients spanning the period from 2012 to 2014 ($26,304$ data points per client). We utilize the curated version of this dataset provided by~\cite{godahewa2021monash}. Our objective is to build a forecaster for a single user. Such models can be used by local grid operators to predict loads from large consumers.

The forecaster is trained to predict hourly electricity consumption $12$ hours ahead using data from the past $84$ hours,
and the training results are presented in Fig.~\ref{fig:res:2}. The results demonstrate a Mean Absolute Percentage Error (MAPE) of $6.53\%$, indicating strong forecasting accuracy. The efficacy of the \emph{Spacetime} model on other benchmark datasets is depicted in the appendix. 

\begin{figure}[t!]
    \centering
    \includegraphics[width=1\linewidth]{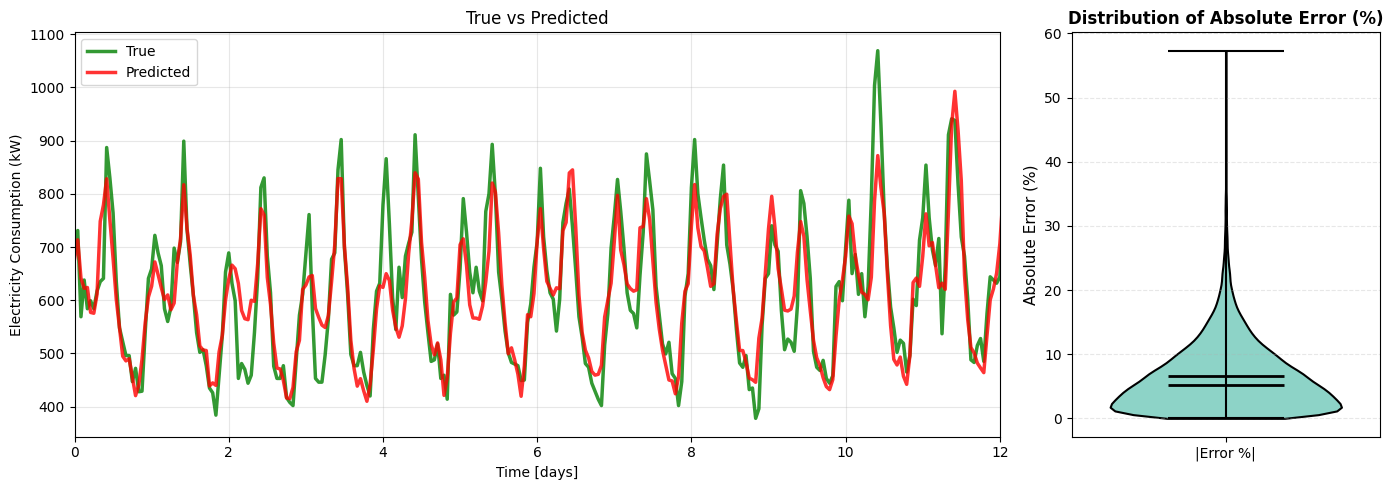}
    \caption{Forecaster performance on test data excerpt (left) and distribution of absolute percentage errors (right). Mean Absolute Percentage Error: $6.53\%$.}
    \vspace{-15pt}
    \label{fig:res:2}
\end{figure}
%
%
%
%
\section{Robust forecaster design}\label{sec:design}
In this section, we present the attack model, a control-theoretic bound on the adversarial error, the adversarial training procedure for solving \eqref{eq:problem}, and experimental validation.

\subsection{Attack Model}
In this paper, we use Projected Gradient Descent \cite{madry2018towards} to generate adversarial attacks against the forecaster, as it is computationally scalable. Additionally, the \emph{Spacetime} forecaster is continuously differentiable, as it is composed of linear state-space operations, GeLU activations, and affine transformations, all of which are smooth. This enables gradient-based attack generation.

Our attack generation method is described in Algorithm~\ref{alg:PGD}. The adversary generates attacks that increase the prediction error by perturbing the input in the direction that maximizes the loss~\eqref{opt:adversary}, determined via the gradient. The attack iteration is terminated when the detection statistic exceeds the threshold $\delta$, so the attacks are stealthy by construction.

Having described the PGD attack, we now assume access to an oracle \emph{GD} that produces optimal attacks against the forecaster model $f(\cdot)$, denoted by $a^{\star} = \text{\emph{GD}}(f,y)$.
%
%
%
%
%
%
%
%
\subsection{Robust forecasters via adversarial training} 
The objective of this paper is to construct a robust forecaster by solving the optimization problem~\eqref{eq:problem}. We achieve this through adversarial training, described in Algorithm~\ref{alg:AT}. The overall approach is as follows: we begin by training a forecaster on the clean dataset, generate optimal attacks against this forecaster using Algorithm~\ref{alg:PGD}, and then fine-tune the model on adversarial inputs with clean targets.

In this paper, we consider two forms of detectors. First, we consider an autoencoder-based detector where the reconstruction error serves as the detection statistic; the function $g(\cdot,\cdot)$ in~\eqref{eq:detector} represents the reconstruction error of the input:
\begin{equation}
    g(\tilde{y}_k, \bar{y}_k) = \|D(E(\tilde{y}_k)) - \tilde{y}_k\|_2,
\end{equation}
where $E(\cdot)$ and $D(\cdot)$ are the CNN encoder and decoder, respectively. Such autoencoder-based detectors are widely adopted in the machine learning community~\cite{yin2020anomaly} and, being data-driven, can be easily integrated into our framework.

Second, we consider an error-norm-based detector where $g(\cdot,\cdot)$ in~\eqref{eq:detector} represents the norm of the prediction error:
\begin{equation}\label{det:power}
    g(\tilde{y}_k, \bar{y}_k) = \|\tilde{y}_k - \bar{y}_k\|_2 = \|e_k\|_2.
\end{equation}
Such detectors are popular in the control theory literature. Since most detectors in control theory are model-based \cite{anand2025feasibility} and assume stationarity, norm-based detectors are among the few model-free alternatives available and have been well studied for power grid applications~\cite{liu2011false}. We next present the experimental results on the benchmark dataset.

\begin{algorithm}[t!]
\caption{GD Attacks with Output Error Constraint}
\label{alg:PGD}
\begin{algorithmic}[1]
\algnotext{EndFor}
\algnotext{EndIf}
\Require Forecaster $f$, input $y$, target $\bar{y}$, step size $\eta$, max iterations $M$, alarm threshold $\delta$, and function $g$ in \eqref{eq:att:pre}
\State Initialize: $a^{(0)} = 0$
\For{$t = 0, 1, \ldots, M-1$}
    \State Compute prediction: $\hat{y}^{(t)} = f(y + a^{(t)})$
    \State Compute gradient: $g^{(t)} = \nabla_a L^{(t)}$, $L^{(t)} = \|\hat{y}^{(t)} - \bar{y}\|_2$
    \State Candidate update: $a^{(t+1)} = a^{(t)} + \eta \cdot g^{(t)}/\|g^{(t)}\|_2$
    \State Update predictions: $\hat{y}^{(t+1)} = f(y + a^{(t+1)})$
    \If{$ g(\hat{y}^{(t+1)}, \bar{y}) > \delta$}
        \State $a^{(t+1)} = a^{(t)}$ and \textbf{break}
    \EndIf
\EndFor
\State \Return $a^{(*)}$
\end{algorithmic}
\end{algorithm}

\begin{algorithm}[t!]
\caption{Adversarial Training}
\label{alg:AT}
\begin{algorithmic}[1]
\algnotext{EndFor}
\Require Clean dataset $\mathcal{I}$, batch size $L$, max iterations $T$
\State Train forecaster on clean data: $f \leftarrow \text{\emph{Spacetime}}(\mathcal{I},\cdot)$
\For{$t = 1, \ldots, T$}
    \State Select $L$ windows $\{y_{j_1}, \ldots, y_{j_L}\}$ from $\mathcal{I}$
    \State Generate attacks: $a_{j_i}^{\star} = \text{\emph{GD}}(f, y_{j_i})$ for $i = 1, \ldots, L$
    \State Fine-tune $f$ on adversarial inputs with clean targets
\EndFor
\State \Return Robust forecaster $f$
\end{algorithmic}
\end{algorithm}
\subsection{Experiments}
In this section, we demonstrate the performance of the robust forecaster on the dataset introduced in Section~\ref{dataset:introduce}.
\subsubsection{Autoencoder-Based Detector}
In this subsection, we use a CNN autoencoder-based detector with an encoding dimension of $2$. The encoder consists of two convolutional layers with filter sizes of $8$ and $16$, respectively, each followed by ReLU activation and max-pooling with a stride of $2$. The convolutional layers use a kernel size of $3$ with padding to preserve spatial dimensions before pooling. The flattened output is then compressed to a $2$-dimensional encoding via a fully connected layer. The decoder mirrors this architecture using transposed convolutions with kernel size $4$ and stride $2$ to upsample the signal back to the original sequence length. The reconstruction error, computed as the mean squared error between the input and reconstructed sequences, serves as the detection statistic. We set a detector threshold of $\delta=0.98$ such that there are $4$ to $5$ false alarms per month.

For attack generation, we use Algorithm~\ref{alg:PGD} and the step size $\eta \approx 10^{-2}$ is chosen so that the threshold is reached during attack generation. For adversarial training, we use Algorithm~\ref{alg:AT} with a batch size of $L = 1000$ ($3\%$ of training data).
After adversarial training, the model is evaluated on $N=50$ different data points (uniformly distributed across a year), and the results are given in Table~\ref{tab:adv:comparison}.
\subsubsection{Norm-Based Detector}
We use a detector of the form \eqref{eq:att:pre}, \eqref{det:power} with $\delta = 400$, such that there are approximately $2$ to $3$ false alarms per year. For attack generation, we use Algorithm~\ref{alg:PGD} and the results are given in Table~\ref{tab:adv:comparison}.

\subsubsection{Discussion}
To compare the detectors, we compute the adversarial MAE per unit attack norm, obtaining $\text{MAE}/\|a\| = 0.35$ for the CNN-based detector and $0.45$ for the norm-based detector. In other words, the adversary achieves greater (normalized) forecasting error against the norm-based detector. Thus the CNN-based detector provides stronger robustness guarantees against stealthy attacks. The fine-tuned model exhibits improved performance on clean data which can be attributed to reduced overfitting. 

We also include a classical input-constrained adversarial baseline, where PGD attacks in Algorithm~\ref{alg:PGD} are clipped to a threshold (instead of constraining the outputs). As shown in Table~\ref{tab:adv:comparison}, the CNN-based detector yields a more robust model ($\approx 10\%$ improvement). This indicates that detector-constrained adversarial training can yield more robust models; however, we note that the attack formulation in these setups are fundamentally different (see Remark~\ref{rem:energy}). Experimental results on other benchmark datasets against a CNN detector are presented in the appendix.

\begin{table}[t]
\centering
\caption{Model Robustness Evaluation across Different Detectors}
\label{tab:adv:comparison}
\resizebox{\columnwidth}{!}{%
\begin{tabular}{lcccc}
\hline
&  {Clean} & Attack & {Adv.} & {\tiny \underline{Adv. MAE}}\\
Model &  {MAE}  & Norm & {MAE} & ${\|a\|}$ \\
\hline
Baseline           & $48.18$   & $-$ & $-$ & $-$ \\ \hline
Fine-tuned (CNN)   & $46.36$   & $634.49$ &$224.86$ & {\cellcolor{green!25}$\mathbf{0.35}$} \\
Fine-tuned (Norm)  & $46.20$   & $236.09$ &$107.87$ & $0.46$ \\
\hline
Fine-tuned (classical)   & $46.35$  & $397.04$& $153.91$ & $0.39$ \\
\hline
\end{tabular}%
}
\vspace{-15pt}
\end{table}

\subsection{A control-theoretic view of \emph{Spacetime} model robustness}
In this section, we quantify the deviation caused by adversarial perturbations and identify which network components contribute most to the prediction error, providing insights for robust forecaster design. We first present the theoretical analysis followed by simple experiments.

\subsubsection{Theory}
While the robust forecaster design in~\eqref{eq:problem} considers detector-constrained adversaries, the following analysis considers a unit-norm input perturbation to characterize the sensitivity of the \emph{Spacetime} model to adversarial inputs and identify which network components amplify vulnerability. We now present the main result of this section, from which we derive key observations.
\begin{proposition}\label{prop:error}
Consider a \emph{Spacetime} model with one \emph{Spacetime} layer in the encoder and the decoder. Suppose the MLP layers act as identity functions. Let $u \in \mathbb{R}^\ell$ be the input vector and $\varepsilon \in \mathbb{R}^\ell$ a perturbation with $\|\varepsilon\|_2 \leq 1$. Then it holds that
\begin{equation}\label{eq:bound:SV}
    \sup_{\|\varepsilon\|_2 \leq 1} \|\tilde{y} - y\|_2 = \sigma_{\max}(H), 
\end{equation}
and the optimal perturbation vector $\varepsilon^\star$ is the right singular vector corresponding to $\sigma_{\max}(H)$. Here $y=Hu$ are the forecasts in the absence of perturbations, $\tilde{y} = H (u+\varepsilon)$ are the forecasts under perturbations. It also holds that
\begin{equation}\label{eq:bound}
    \left( 1/\sqrt{h}\right)\|H\|_1 \leq 
    \sup_{\|\varepsilon\|_2 \leq 1} \|\tilde{y} - y\|_2 
    \leq \sqrt{\ell}\|H\|_1,
\end{equation}
where $H$ is the input-output map defined as 
\begin{equation}\label{eq:linear:ST}
H_{i,j} = \bar{C}(\bar{A}+\bar{B}\bar{K})^{i} \sum_{k=j}^{\ell-1} \left( \bar{A}^{\ell-1-k}\bar{B} \, CA^{k-j}B \right), 
\end{equation}
where $i \in \{0,1,\dots,h-1\}$ and $j \in \{0,1,\dots,\ell-1\}$. 
\end{proposition}
\begin{proof}
The encoder is characterized by state-space matrices $(A,B,C)$ with $A\in \mathbb{R}^{n_e \times n_e}$, while the decoder has 
matrices $(\bar{A},\bar{B},\begin{bmatrix}\bar{C}\\\bar{K}\end{bmatrix})$ with $\bar{A} \in \mathbb{R}^{n_d \times n_d}$. Here, $\bar{K}$ denotes the feedback matrix that predicts the encoded inputs. The relation $y=Hu$ follows from the structure of the \emph{Spacetime} model under the stated assumptions. Since $\|\tilde{y} - y\|_2 = \|H\varepsilon\|_2$ follows from the linearity of $H$, the exact supremum $\sup_{\|\varepsilon\|_2 \leq 1}\|H\varepsilon\|_2 = \sigma_{\max}(H)$ follows from the definition of the spectral norm, with the optimal $\varepsilon^\star$ given by the corresponding right singular vector. The bounds in~\eqref{eq:bound} follow from the inequality $\frac{1}{\sqrt{m}}\|M\|_1 \leq \|M\|_2 \leq \sqrt{n}\|M\|_1$ for $M \in \mathbb{R}^{m \times n}$. This concludes the proof.
\end{proof}

Although \eqref{eq:bound:SV} characterizes the exact adversarial error and the optimal attack vector, it does not reveal how individual components of the network contribute to vulnerability. To this end, note that $\|H\|_1$ can be reformulated as
\begin{equation}
\scalebox{0.9}{$\displaystyle\max_{j \in \{0,\dots,\ell-1\}} \left\{ \sum_{i=0}^{h-1} \left\vert \bar{C}(\bar{A}+\bar{B}\bar{K})^{i} \sum_{k=j}^{\ell-1} \left( \bar{A}^{\ell-1-k}\bar{B} \, CA^{k-j}B \right) \right\vert \right\}.$}
\end{equation} 
We now make several important observations.

\textbf{Observation 1 (Open-loop instability amplifies long-lag errors):} If the encoder matrix $A$ or the open-loop decoder matrix $\bar{A}$ is unstable (i.e., $\rho(A) > 1$ or $\rho(\bar{A}) > 1$), then terms involving $A^{\ell-1}$ or $\bar{A}^{\ell-1}$ can dominate $\|H\|_1$. For long input sequences (large $\ell$), these terms grow exponentially, causing the adversarial error bound to increase exponentially with $\ell$. Thus, open-loop stability of both encoder and decoder is critical for robustness when using long look-back sequences.

\textbf{Observation 2 (Closed-loop instability amplifies long-horizon errors):} If the closed-loop decoder matrix $(\bar{A}+\bar{B}\bar{K})$ is unstable (i.e., $\rho(\bar{A}+\bar{B}\bar{K}) > 1$), then $\|(\bar{A}+\bar{B}\bar{K})^{i}\|$ grows exponentially with $i$. For long prediction horizons, the terms with large $i$ dominate $\|H\|_1$, causing the adversarial error bound to increase exponentially with the forecast horizon $h$. Thus, closed-loop stability is critical for robustness in long-horizon forecasting.

\textbf{Observation 3 (Decoder dimension-dependent scaling):} 
The matrix $H$ map can be written as $H = H_1 H_2$, where
\begin{equation}
\scalebox{0.9}{$H_1 = \begin{bmatrix}
\bar{C}^\top & 
\left(\bar{C} \left( \bar{A}+\bar{B}\bar{K} \right)\right)^\top & 
\dots &
\left(\bar{C}\left(\bar{A}+\bar{B}\bar{K} \right)^{h-1}\right)^\top
\end{bmatrix}^\top,$}
\end{equation}
and $H_2 \in \mathbb{R}^{n_d \times \ell}$ is a matrix whose columns are given by
$(H_{2})_{:,j} = \displaystyle \sum_{k=j}^{\ell-1} \bar{A}^{\ell-1-k}\bar{B} \, CA^{k-j}B$.
Using submultiplicativity of the spectral norm and the inequality $\|H_2\|_2 \leq \sqrt{n_d} \|H_2\|_{\infty}$, we obtain
\begin{align}\label{eq:bound:O4}
\scalebox{0.95}{$\|H\|_2 \leq \|H_1\|_2 \|H_2\|_2 \leq \|H_1\|_2 \sqrt{n_d} \|H_2\|_{\infty}.$}
\end{align}
Thus, the decoder state dimension plays a non-trivial role in the model's sensitivity to adversarial perturbations; however, this bound is conservative 
and may not be tight in practice. Finally, note that the adversarial analysis extends naturally to other deep SSMs, and is not 
exclusive to \emph{Spacetime}.

\subsubsection{Experimental Validation}
We validate Observations~1 and~2 experimentally using a simplified linear predictor of the form~\eqref{eq:linear:ST}. The target signal is a noisy sine wave, and the model is trained to minimize the mean squared error.

Our goal is to show that the adversarial error grows with $\ell$ and $h$ when the spectral radius of the encoder matrix $A$ and closed-loop decoder matrix $\bar{A}+\bar{B}\bar{K}$, respectively, are near or above unity. Since the spectral radius cannot be fixed prior to training, we train separate models for different values of $\ell$ and $h$ and observe the spectral radii post-training. When varying $h$ (with $\ell$ fixed), the encoder spectral radius and the closed-loop decoder spectral radius remains in the range $[1.0378, 1.0614]$ and $[0.9978, 1.0027]$, respectively. Similarly, when varying $\ell$ (with $h$ fixed), the encoder spectral radius and the closed-loop decoder spectral radius remains in the range $[0.8742, 0.9274]$ and $[1.0921, 1.1631]$, respectively. This confirms that the spectral radii remain approximately constant across models, ensuring they are not confounding variables. For each model, we construct input constrained attacks using PGD and plot the adversarial error in Fig.~\ref{fig:obs}. As shown, the adversarial error increases monotonically in both cases, consistent with Observations~1 and~2.

We also observe errors of $4.04$, $4.10$, and $4.11$ for $n_d = 2, 3, 6$, respectively, consistent with the trend predicted by Observation~3, though the bound remains conservative. Finally, we note that in this experiment (with $\ell=3$ and $H=10$), the trained encoder is found to possess an unstable zero of magnitude $1.39$, consistent with Remark~\ref{rem:zero_dynamics}.

\begin{figure}
    \centering
    \includegraphics[width=8cm]{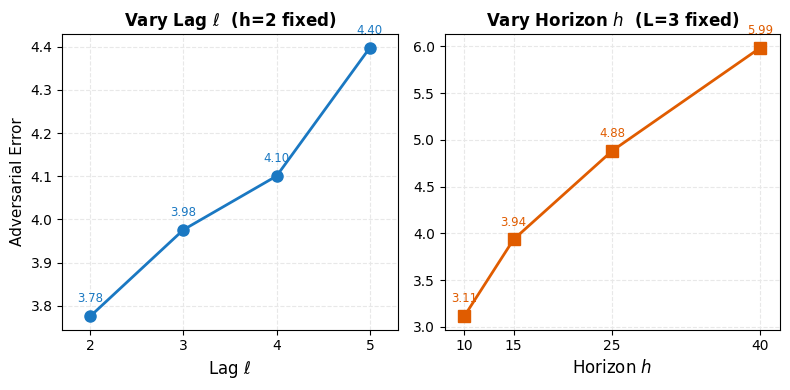}
    \vspace{-10pt}
    \caption{Adversarial error as a function of $\ell$ (left) and $h$ (right), with approximately constant spectral radius across models in both experiments.}
    \vspace{-15pt}
    \label{fig:obs}
\end{figure}
\section{Model-free attacks}\label{sec:DDA}
In the previous sections, we assumed that the adversary has access to the forecaster model to construct attacks. In this section, we show that even without access to the forecaster, an adversary can construct stealthy attacks using only data against a norm-based detector.
\subsection{Data-driven attacks (DDAs): Theory}
Before presenting the main result, we recap some notation. Let us denote the forecaster input as $y_i$, the target as $y_o$, and the predicted output as $\hat{y}_o = f(y_i)$. Let the prediction error be $e = \hat{y}_o - y_o$, and an alarm is raised when $\|e\|_2 > \delta$, where $\delta$ can be tuned to enforce \eqref{det:power}. As \emph{Spacetime} exhibits locally linear behavior due to their linear state-space operations and smooth activations, we can reasonably assume the forecaster approximately preserves the norm ratio $\gamma = \frac{\|\hat{y}_o\|_2}{\|y_i\|_2}$ and directional alignment $\beta = \frac{\langle \hat{y}_o, y_o \rangle}{ \|\hat{y}_o\|_2 \|y_o\|_2}$ under small input perturbations. Then we have the following result.

\begin{theorem}\label{thm:DDA}
Let the attack vector $a$ be designed such that:
\begin{equation}\label{eq:DDA}
    \|\tilde{y}\|_2 = \gamma^{-1} \left( \beta \|y_o\|_2 \pm \sqrt{\|y_o\|_2^2 (\beta^2 - 1) + \mu}\right)
\end{equation}
\begin{equation}
\text{with}\; \gamma = \frac{\|\hat{y}_o\|_2}{\|y_i\|_2}, \quad \beta = \frac{\langle \hat{y}_o, y_o \rangle}{\|\hat{y}_o\|_2 \|y_o\|_2}, \quad \mu = (\delta - s)^2,
\end{equation}
where $\hat{y}_o = f(y_i)$ and $s > 0$ is a slack variable. Then the attack $a$ is stealthy with $\|e\|_2 \leq \delta$.
\end{theorem}

\begin{proof}
A sufficient condition for the attack to be stealthy is $\|e\|_2 = \delta - s$, where $s > 0$ is a slack term. Let us reformulate $\|e\|_2 = \delta - s$ as:
\begin{align}
    \|e\|_2^2 = \|\hat{y}_o\|_2^2 + \|y_o\|_2^2 - 2\langle \hat{y}_o, y_o \rangle = (\delta - s)^2 &= \mu\\
    \implies \;\|\hat{y}_o\|_2^2 - 2\|\hat{y}_o\|_2 \|y_o\|_2 \beta + \|y_o\|_2^2 - \mu &= 0 \label{eq:1}
\end{align}
where we used $\langle \hat{y}_o, y_o \rangle = \beta \|\hat{y}_o\|_2 \|y_o\|_2$. Solving for $\|\hat{y}_o\|_2$ from \eqref{eq:1} using the quadratic formula gives:
\begin{equation}
    \|\hat{y}_o\|_2 = \beta \|y_o\|_2 \pm \sqrt{\|y_o\|_2^2 (\beta^2 - 1) + \mu}
\end{equation}
Given the forecaster approximately preserves the norm ratio $\gamma$, it follows that the attacked input should satisfy \eqref{eq:DDA}.
\end{proof}

Theorem~\ref{thm:DDA} states that for a locally linear model, the attack vector can be explicitly designed to satisfy \eqref{eq:DDA}. Note that to design attacks using \eqref{eq:DDA}, we do not need access to the forecaster model. We only require the forecaster input $y_i$, the alarm threshold $\delta$, and the target $y_o$, which are the same information available in Algorithm~\ref{alg:PGD} except for the model itself. Also, the stealthiness guarantee in Theorem~\ref{thm:DDA} is agnostic to the adversary's objective and the  holds for any attack direction. Theorem~\ref{thm:DDA} only requires estimates of the local gain $\gamma$ and the alignment coefficient $\beta$. We next discuss how to obtain these values in a data-driven fashion.

For a well-trained forecaster with small prediction error $\|y_o - \hat{y}_o\|_2$, the triangle inequality gives $|\|\hat{y}_o\|_2 - \|y_o\|_2| \leq \|y_o - \hat{y}_o\|_2$, implying $\|\hat{y}_o\|_2 \approx \|y_o\|_2$. Therefore, the adversary can approximate $\gamma \approx \|y_o\|_2 / \|y_i\|_2$ using only the target data. Similarly, for a well-trained forecaster, we have $\hat{y}_o \approx y_o$, which implies:
$\beta = \frac{\langle \hat{y}_o, y_o \rangle}{\|\hat{y}_o\|_2 \|y_o\|_2} \approx \frac{\langle y_o, y_o \rangle}{\|y_o\|_2^2} = 1.$
The adversary can thus use the approximation $\beta \approx 1$, or choose a conservative lower bound to account for prediction inaccuracies. Alternatively, $\beta$ can be estimated from a few queries to the forecaster if limited access is available. Thus, DDA described in Theorem~\ref{thm:DDA} requires no gradient information, making it practical even in black-box scenarios.

We note that an attack is feasible only if $\|y_o\|_2^2 (\beta^2 - 1) + \mu \geq 0$ to ensure the square root in \eqref{eq:DDA} is real. Since $\beta \leq 1$, the condition simplifies to $(\delta - s)^2 \geq \|y_o\|_2^2 (1 - \beta^2) \geq 0$. This is satisfied when $\delta$ is sufficiently large relative to $\|y_o\|_2$ or when $\beta$ is close to $1$ (good forecaster). We next demonstrate the efficacy of the proposed DDA method on the electricity consumption benchmark described previously. We also demonstrate the efficacy of the DDAs on other benchmark datasets in the appendix, confirming that the approximations $\gamma \approx \|y_o\|_2/\|y_i\|_2$ and $\beta \approx 1$ are not overly conservative and generalize across diverse time series domains.
\subsection{Data-driven attacks: Experiments}
Let us consider the electricity consumption dataset described in Section~\ref{dataset:introduce}. The attack vector is constructed in the normalized input direction ($y_i$), scaled to satisfy the bound in \eqref{eq:DDA}. We use $\gamma$ estimated from data as explained previously and we use $\beta = 0.9907$. The results are presented in Fig.~\ref{fig:DDA}, which shows the distribution of error over $50$ samples. We observe that the DDAs achieve an error in the range of PGDs with small and moderate normalized step sizes ($\eta=10^{-5}$ and $\eta=10^{-2}$), without any gradient computation. While PGD requires careful step size selection and model access, DDA achieves competitive error without any gradient computation, making it significantly more efficient.

We finally note that a well-trained \emph{Spacetime} forecaster satisfies the condition under which Theorem~\ref{thm:DDA} guarantees a stealthy attack. This reveals an inherent tension: the locally linear structure that makes SSM-based forecasters accurate predictors also makes them susceptible to model-free attacks.

\begin{figure}
    \centering
    \includegraphics[width=5cm]{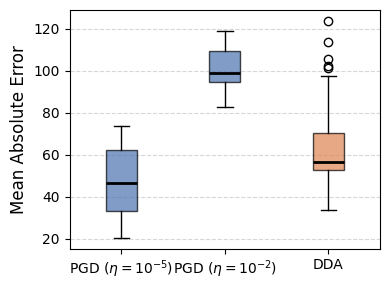}
    \vspace{-10pt}
    \caption{Adversarial error caused by PGD attacks, and data-driven attacks.}
    \vspace{-15pt}
    \label{fig:DDA}
\end{figure}

\begin{table*}
\centering
\caption{Robustness evaluation across Monash benchmark datasets in \cite{godahewa2021monash}}
\label{tab:adv:datasets}
\scalebox{0.85}{
\begin{tabular}{||c||c|c|c|c||c|c||c|c|c|c||}
\hline
\;& \multicolumn{2}{c|}{\textbf{Baseline (CNN Det.)}} & \multicolumn{2}{c||}{\textbf{Fine-tuned (CNN Det.)}} & \multicolumn{2}{c||}{\textbf{Baseline (Norm Det.)}} & \multicolumn{4}{c||}{\textbf{Detector setup}} 
\\
\hline
\; & {Clean} & {Adv.} & {Clean} & {Adv.} & {PGD} & {DDA} & \multicolumn{2}{c|}{Threshold} & {Encoder} & {Encoder}\\
\textbf{Dataset} & {MAE} & {MAE} & {MAE} & {MAE} & {MAE} & {MAE} & $\delta_{\rm norm}$ & $\delta_{\rm CNN}$ & {layers} & {dim.} \\
\hline
S.F. Traffic & {\cellcolor{green!25}$0.01206$} & $0.06922$ & $0.01274$ & {\cellcolor{green!25}$0.06549$} & $0.0692$ & {\cellcolor{green!25}$0.2607$} & $1$ & $1.8$ & $4/8$ & $6$ \\ \hline
River Flow & $10.62$ & $12.06$ & {\cellcolor{green!25}$7.35$} & {\cellcolor{green!25}$10.10$} & $8.89$ & {\cellcolor{green!25}$12.43$} & $10$ & $0.9$ & $8/16$ & $2$\\ \hline
U.S. Births & $379.1$ & $380.3$ & {\cellcolor{green!25}$366.1$} & {\cellcolor{green!25}$367.2$} & $379.5$ & {\cellcolor{green!25}$922.4$} & $4500$ & $0.8$ & $4/8$ & $3$ \\
\hline
\end{tabular}
}
\vspace{-10pt}
\end{table*}

\section{Conclusions}\label{sec:con}
In this paper, we studied the adversarial robustness of the \emph{Spacetime} model. We formulated a robust optimization problem against worst-case stealthy adversaries, solved via adversarial training. We further characterized the dependence of forecasting error on model parameters, providing insights for robust forecaster design. Finally, we demonstrated on benchmark datasets, that attacks can be easily constructed without knowledge of the forecaster, underscoring the vulnerability of SSM-based forecasters. 
Future work includes extending the framework to study targeted adversarial attacks where the adversary steers the forecaster towards a specific false prediction rather than simply maximizing the MSE.
\bibliographystyle{ieeetr}
\bibliography{References}
\appendix
\section{}
The results obtained on three additional Monash benchmark datasets~\cite{godahewa2021monash} are presented in Table~\ref{tab:adv:datasets}. 
Fine-tuning is performed with around $1\%$ of attacked data. The clean (adversarial) MAE represents the forecasting performance in the absence of attacks (under PGD attacks in Algorithm~\ref{alg:PGD}). For all models, the evaluation is done with $N=50$ samples, $\ell=84$, and $h=12$. The DDA MAE represents the forecasting error under data-driven attacks where we use $\beta=0.9907$ across all datasets, against a norm-based detector with threshold $\delta_{\rm norm}$, and $\delta_{\rm CNN}$ denotes the threshold for the CNN-based detector. The encoder architecture is described by the number of layers and encoding dimension.

The highlighted columns in Table~\ref{tab:adv:datasets} summarize the key findings. On the left, adversarial fine-tuning consistently reduces the adversarial MAE, demonstrating improved robustness. In some datasets, fine-tuning also improves the clean MAE, which can be attributed to a regularization effect that reduces overfitting. Here, to ensure a fair comparison between the baseline and fine-tuned models, the attack norm $\|a\|$ is identical, ensuring that the reported improvements in adversarial MAE reflect genuine robustness gains rather than differences in attack strength. On the right, the highlighted DDA MAE entries show that model-free attacks can induce significant forecasting errors compared to PGD. The adversarial errors under the norm-based detector are higher than those under the CNN detector, as the thresholds $\delta_{\rm norm}$ and $\delta_{\rm CNN}$ are not tuned to yield equal FAR.
\end{document}